\documentclass[twoside,11pt]{article}

\usepackage{jmlr2e}

\usepackage{graphicx}
\usepackage{subfigure}
\usepackage{amssymb}
\usepackage{amsmath}
\usepackage{amsbsy}
\usepackage{algorithm}
\usepackage{algorithmic}
\usepackage{multirow}

\ShortHeadings{Local Rademacher Complexity for Multi-label Learning}{}
\firstpageno{1}

\begin{document}

\title{Local Rademacher Complexity for Multi-label Learning}

\author{\name Chang Xu \email xuchang@pku.edu.cn \\
 \addr  Key Laboratory of Machine Perception (Ministry of Education)\\
       School of Electronics Engineering and Computer Science\\
       Peking University\\
        Beijing 100871, China
      \AND
       \name Tongliang Liu \email tongliang.liu@student.uts.edu.au \\
       \addr  Centre for Quantum Computation and Intelligent Systems\\
       Faculty of Engineering and Information Technology\\
       University of Technology, Sydney\\
       Sydney, NSW 2007, Australia
       \AND
       \name Dacheng Tao \email dacheng.tao@uts.edu.au \\
       \addr  Centre for Quantum Computation and Intelligent Systems\\
       Faculty of Engineering and Information Technology\\
       University of Technology, Sydney\\
       Sydney, NSW 2007, Australia
       \AND
       \name Chao Xu \email xuchao@cis.pku.edu.cn\\
       \addr  Key Laboratory of Machine Perception (Ministry of Education)\\
       School of Electronics Engineering and Computer Science\\
       Peking University\\
        Beijing 100871, China
}


\maketitle

\begin{abstract}
We analyze the local Rademacher complexity of empirical risk minimization (ERM)-based multi-label learning algorithms, and in doing so propose a new algorithm for multi-label learning. Rather than using the trace norm to regularize the multi-label predictor, we instead  minimize the tail sum of the singular values of the predictor in multi-label learning. 
Benefiting from the use of the local Rademacher complexity, our algorithm, therefore,  has a sharper generalization error bound and a faster convergence rate.
Compared to methods that minimize over all singular values,  concentrating on the tail singular values  results in better recovery of the low-rank structure of the multi-label predictor, which plays an import role in exploiting label correlations. We propose a new conditional singular value thresholding algorithm to solve the resulting objective function. Empirical studies on real-world datasets validate our theoretical results and demonstrate the effectiveness of the proposed algorithm.
\end{abstract}


\section{Introduction}

In multi-label learning, an example can be assigned  more than one label. This is different to the conventional single label learning, in which each example corresponds to one, and only one, label. Over the past few decades, multi-label learning \citep{zhang2013review} has been successfully applied to many real-world applications such as text categorization \citep{gao2004mfom}, image annotation \citep{wang2009multi}, and gene function analysis \citep{barutcuoglu2006hierarchical}.

A straightforward approach to multi-label learning is to decompose it into a series of binary classification problems for different labels\citep{tsoumakas2010mining}. However, this approach can result in poor performance  when strong label correlations exist. To improve prediction, a large number of algorithms have been developed that approach the multi-label learning problem from different perspectives such as the classifier chains algorithm \citep{read2011classifier}, the max-margin multi-label classifier \citep{hariharan2010large}, the probabilistic multi-label learning algorithms \citep{zhang2010multi,guo2013probabilistic}, the correlation learning algorithms \citep{huang2012multi,bi2014multilabel}, and label dependency removal  algorithms \citep{chen2012feature,tai2012multilabel}.


Vapnik's learning theory \citep{vapnik1998statistical} can be used to justify the successful development of multi-label learning \citep{zhi2013multi,li2013active,doppa2014hc}. 
One of the most useful data-dependent complexity measures is Rademacher complexity \citep{bartlett2003rademacher}, which leads to tighter generalization error bounds than those derived using the VC dimension and cover number. Recently, \citep{yu2013large} proved that the Rademacher complexity of empirical risk minimization (ERM)-based multi-label learning algorithms  can be bounded by the trace norm of the multi-label predictors, which provides a theoretical explanation for the effectiveness of using the  trace norm for regularization in multi-label learning.
On the other hand, minimizing the trace norm over the predictor  implicitly exploits the correlations between different labels in multi-label learning.

One shortcoming of the general Rademacher complexity is that it ignores the fact that the hypotheses selected by a learning algorithm usually 
belong to a more favorable subset of all the hypotheses, and they therefore have better performance than in the worst case.  To overcome this drawback, the local Rademacher complexity considers the  Rademacher averages of smaller subsets of the hypothesis set. This results in a sharper generalization error bound than that derived using global Rademacher complexity. Specifically, the generalization error bound derived by Rademacher complexity is at most of convergence order of $\mathcal{O}(\sqrt{1/n})$, while the bound obtained using local Rademacher complexity usually converges as fast as $\mathcal{O}(\log{n}/n)$. We therefore seek to use  local Rademacher complexity in multi-label learning problem and  design a new algorithm.

In this paper, we show that the local Rademacher complexities for ERM-based multi-label learning algorithms can be upper-bounded in terms  of the tail sum of the singular values of the multi-label predictor. As a result, we are motivated  to penalize the tail sum of  the singular values of the multi-label predictor in multi-label learning, rather than the sum of all the singular values (i.e., the trace norm). As well as  the advantage of producing a sharper generalization error bound, this new constraint over the multi-label predictor   achieves  better recovery of the low-rank structure of the predictor and  effectively exploits the correlations between labels in multi-label learning.
The resulting objective function can be efficiently solved using a newly proposed conditional singular value thresholding algorithm. Extensive experiments on real-world datasets validate our theoretical analysis and demonstrate the effectiveness of the new multi-label learning algorithm. 

\section{Global and Local Rademacher Complexities}


In a standard supervised learning setting, a set of training examples $z_{1}=(x_{1}, y_{1}), \cdots, z_{n}=(x_{n}, y_{n})$ are i.i.d. sampled from distribution $\mathcal{P}$ over $\mathcal{X}\times \mathcal{Y}$. Let $\mathcal{F}$ be a set of functions mapping  $\mathcal{X}$ to $\mathcal{Y}$. The learning problem is to select a function $f\in\mathcal{F}$ such that the expected loss $\mathbb{E}[\ell(f(x),y)]$ is small, where $\ell(\cdot): \mathcal{Y}\times\mathcal{Y}\rightarrow [0,1]$ is a loss function. Defining $\mathcal{G} = \ell(\mathcal{F},\cdot)$ as the loss class, the learning problem is then equivalent to findings a function $g\in\mathcal{G}$ with small $\mathbb{E}[g]$. 

Global Rademacher complexity \citep{bartlett2003rademacher} is an effective approach for measuring the richness (complexity) of the function class $\mathcal{G}$, and it is defined as
\begin{definition}
Let $\sigma_1, \cdots, \sigma_{n}$ be independent uniform ${-1, 1}$-valued random variables.  The global Rademacher complexity of $\mathcal{G}$ is then defined as
\begin{equation}
R_{n}(\mathcal{G}) = \mathbb{E}\left[ \sup_{g\in\mathcal{G}}\frac{1}{n} \sum_{i=1}^{n}\sigma_{i}g(z_{i}) \right].
\end{equation}
\end{definition}
Based on the notion of global Rademacher complexity, the algorithm has a standard generalization error bound, as shown in the following theorem \citep{bartlett2003rademacher} .
\begin{theorem}\label{the:global_bound}
Given $\delta>0$, suppose the function $\widehat{g}$ is learned over $n$ training points. Then, with probability at least $1-\delta$, we have
\begin{equation}
\mathbb{E}[\widehat{g}] \leq \inf_{g\in \mathcal{G}} \mathbb{E}[g] + 4R_{n}(\mathcal{G}) + \sqrt{\frac{2\log(2/\delta)}{n}}.
\end{equation}
\end{theorem}
Since global Rademacher complexity $R_{n}(\mathcal{G})$ is in the order of $\mathcal{O}(\sqrt{1/n})$ for various classes used in practice, the generalization error bound in Theorem \ref{the:global_bound} converges at rate $\mathcal{O}(\sqrt{1/n})$.
Global Rademacher complexity is a global estimation of the complexity of the function class, and thus it ignores the fact that the algorithm is likely to pick functions with a small error, and, in particular, only a small subset of the function class will be used. 

Instead of using the global Rademacher averages of the entire class as the complexity measure, it is more reasonable to consider the Rademacher complexity of a small subset of the class, e.g.,  the intersection of the class with a ball centered on the function of interest. Clearly, this local Rademacher complexity \citep{bartlett2005local} is always smaller than the corresponding global Rademacher complexity, and its formal definition is given by
\begin{definition}
For any $r>0$, the local Rademacher complexity of $\mathcal{G}$ is defined as
\begin{equation}
R_{n}(\mathcal{G},r) = R_{n}(\{g\in\mathcal{G}: \mathbb{E}[g^{2}]\leq r\}).
\end{equation}
\end{definition}
The following theorem describes the generalization error bound based on local Rademacher complexity.
\begin{theorem}\label{the:local_bound}
Given $\delta>0$, suppose we learn the function $\widehat{g}$ over $n$ training points. Assume that there is some $r>0$ such that for every $g\in\mathcal{G}$, $\mathbb{E}[g^{2}]\leq r$. Then with probability at least $1-\delta$, we have
\begin{equation}
\mathbb{E}[\widehat{g}] \leq \inf_{g\in \mathcal{G}} \mathbb{E}[g] + 8R_{n}(\mathcal{G}) + \sqrt{\frac{8r\log(2/\delta)}{n}} + \frac{3\log(2/\delta)}{n}.
\end{equation}
\end{theorem}
By choosing a much smaller class $\mathcal{G}^{'}\subseteq \mathcal{G}$ with as small a variance as possible while requiring that $\widehat{g}$ still lies in $\mathcal{G}^{'}$, the generalization error bound in Theorem \ref{the:local_bound} has a faster convergence rate than Theorem \ref{the:global_bound} of up to $\mathcal{O}(\log{n}/n)$. Once the local Rademacher complexity is known, $\mathbb{E}[\widehat{g}] - \inf_{g\in \mathcal{G}} \mathbb{E}[g]$ can be bounded in terms of the fixed point of the local Rademacher complexity of $\mathcal{F}$.

\section{Local Rademacher Complexity  for Multi-label Learning}

In this section, we analyze the local Rademahcer complexity for multi-label learning and illustrate our motivation for developing a new multi-label learning algorithm.

The multi-label learning model is described by a distribution $\mathcal{Q}$ on the space of data points and labels $\mathcal{X}\times \{0,1\}^{L}$. We receive $N$ training points $\{(x_{i},y_{i})\}_{i=1}^{N}$ sampled i.i.d. from the distribution $\mathcal{Q}$, where $y_{i}\in \{0,1\}^{L}$ are the ground truth label vectors. Given these training data, we learn  a multi-label predictor $\widehat{W} \in \mathbb{R}^{d\times L}$ by performing ERM as follows:
\begin{equation}\label{eq:erm}
\widehat{W} = \arg \inf_{W\in \phi(W)} \widehat{\mathcal{L}}(W)=\frac{1}{n}\sum_{i=1}^{n} \ell(f(x_{i},W), y_{i}),
\end{equation}
where $\widehat{\mathcal{L}}(W)$ is the empirical risk of  a multi-label learner $W$, and $\phi(W)$ is some constraint on $W$.

\citep{yu2013large} proposes to solve the multi-label learning problem with Eq. (\ref{eq:erm}) by setting $\phi(W)$ as the trace constraint $trace(W)<\lambda$, and then providing its corresponding global Rademacher complexity bound
\begin{equation}
R_{n}(W) \leq \frac{\lambda}{\sqrt{n}}.
\end{equation}
This global Rademacher complexity for multi-label learning is in the order of $\mathcal{O}(\sqrt{1/n})$, which is exactly consistent with the general analysis shown in the previous section. Hence, the generalization error bound based on the global Rademacher complexity in \citep{yu2013large} converges up to  $\mathcal{O}(\sqrt{1/n})$.

In practice, the hypotheses selected by a learning algorithm usually have better performance than the worst case and belong to a more favorable subset of all the hypotheses. Based on this idea, we employ the local Rademacher complexities to measure the complexity of smaller subsets of the hypothesis set, which results in sharper learning bounds and guarantees faster convergence rates. The local Rademacher complexity of the multi-label learning algorithm using Eq. (\ref{eq:erm}) is shown in Theorem \ref{the:main}.

\begin{theorem}\label{the:main}
Suppose we learn $W$ over $n$ training points. Let $W = U\Sigma V$ be the SVD decomposition of $W$, where $U$ and $V$ are the unitary matrices, and $\Sigma$ is the diagonal matrix with singular values $\{\lambda_{i}\}$ in descending order. Assume $\|W\|\leq 1$ and there is some $r>0$ such that for every $W\in \mathcal{W}$, $\|\mathbb{E} [WW^{T}]\|\leq r$. Then, the local Rademacher complexity of $\mathcal{W}$ is
\begin{equation}\nonumber
\begin{split}
\mathbb{E}\left[\sup_{W\in\mathcal{W}}\langle\frac{1}{n}\sum_{i=1}^{n}\sigma_{i}x_{i}, W\rangle\right] \leq & r\sqrt{\frac{\theta}{n}} + \frac{\sum_{j>\theta}\lambda_{j}}{\sqrt{n}} \\
\end{split}
\end{equation}
\end{theorem}

\begin{proof}

Considering $W = U\Sigma V$, $W$ can be rewritten as
\begin{equation}
W = \sum_{j}u_{j}v_{j}^{T}\lambda_{j}, 
\end{equation}
where $u_j$ and $v_j$ are the column vectors of $U$ and $V$, respectively. Based on the orthogonality of $U$ and $V$, we have the following decomposition 
\begin{equation}\nonumber
\begin{split}
\langle\frac{1}{n}&\sum_{i=1}^{n}\sigma_{i}x_{i}, W\rangle =  \langle X_{\sigma}, W\rangle \\
\leq  & \langle \sum_{j=1}^{\theta} X_{\sigma}u_{j}u_{j}^{T}\lambda_{j}^{-1}, \sum_{j=1}^{\theta}u_{j}v_{j}^{T}\lambda_{j}^{2}\rangle + \langle \sum_{j>\theta} X_{\sigma}u_{j}u_{j}^{T}, W \rangle \\
\leq & \| \sum_{j=1}^{\theta} X_{\sigma}u_{j}u_{j}^{T}\lambda_{j}^{-1}\|\|\sum_{j=1}^{\theta}u_{j}v_{j}^{T}\lambda_{j}^{2}\|+\|\sum_{j>\theta} X_{\sigma}u_{j}u_{j}^{T}\|\|W\|.
\end{split}
\end{equation}
Considering
\begin{equation}
\begin{split}
\mathbb{E} &[\| \sum_{j=1}^{\theta} X_{\sigma}u_{j}u_{j}^{T}\lambda_{j}^{-1}\|] = \mathbb{E}  \left[\sqrt{\sum_{j=1}^{\theta}\lambda_{j}^{-2}\langle X_{\sigma}, u_{j} \rangle^{2}}\right] \\
\leq & \sqrt{\sum_{j=1}^{\theta}  \frac{\lambda_{j}^{-2}}{n}\mathbb{E}[\langle x, u_{j} \rangle^{2}]} = \sqrt{\frac{\theta}{n}},
\end{split}
\end{equation}
and
\begin{equation}
\begin{split}
\|\sum_{j=1}^{\theta}u_{j}v_{j}^{T}\lambda_{j}^{2}\| = &\|\sum_{j=1}^{\theta}u_{j}u_{j}^{T}\lambda_{j}^{2}\| \\
\leq & \|\sum_{j=1}^{\infty}u_{j}u_{j}^{T}\lambda_{j}^{2}\|  = \|\mathbb{E}[WW^{T}]\|  \leq r,
\end{split}
\end{equation}
we have
\begin{equation}
\begin{split}
\mathbb{E}\left[\sup_{W\in\mathcal{W}}\langle\frac{1}{n}\sum_{i=1}^{n}\sigma_{i}x_{i}, W\rangle\right] \leq & r\sqrt{\frac{\theta}{n}} + \sqrt{\frac{1}{n}\sum_{j>\theta}\lambda_{j}^{2}}\\
\leq & r\sqrt{\frac{\theta}{n}} + \frac{\sum_{j>\theta}\lambda_{j}}{\sqrt{n}},
\end{split}
\end{equation}
which completes the proof.
\end{proof}

According to Theorem \ref{the:main}, the local Rademacher complexity for ERM-based multi-label learning algorithms is determined by the tail sum of the  singular values. When $\sum_{j>\theta} \lambda_{j}= \mathcal{O}(\exp(-\theta))$, we have $\mathbb{E}[\widehat{g}] - \inf_{g\in \mathcal{G}} \mathbb{E}[g]  = \mathcal{O}(\log{n}/n)$, which leads to a sharper generalization error bound  than that based on global Rademacher complexity.

\section{Algorithm}

In this section,  the properties of the local Rademacher complexity discussed above are used to devise a new multi-label learning algorithm.

Each training point has a feature vector $x_{i}\in \mathbb{R}^{d}$ and a corresponding label vector $y_{i} = \{0,1\}^{L}$. If $y_{ij}=1$, example $x_i$ will have label-$j$; otherwise, there is no label-$j$ for example $x_{i}$. The multi-label predictor is parameterized as $f(x,W) = W^{T}x$, where $W\in \mathbb{R}^{d\times L}$. $\ell(y, f(x,W))\in \mathbb{R}$ is the loss function that computes the discrepancy between the true label vector and the predicted label vector. 

The trace norm is an effective approach for modeling and capturing  correlations between labels associated with examples, and it has been widely adopted in many multi-label algorithms \citep{amit2007uncovering,loeff2008scene,cabral2011matrix,xu2013speedup}. 
Within the ERM framework,  their objective functions usually take the form
\begin{equation}\label{eq:meta}
\min_{W} \sum_{i=1}^{n} \ell(y_{i},f(x_{i},W))+ C\|W\|_{*},
\end{equation} 
where $\|\cdot\|_{*}$ is the trace norm and $C$ is a constant. In particular, for Problem (\ref{eq:meta}), \citep{yu2013large} has proved that the global Rademacher complexity of $W$  is upper-bounded in terms of its trace.

As shown  in the previous section, however, the \emph{tail sum of the singular values} of $W$, rather than its trace, determines the local Rademacher complexity. Since the local Rademacher complexity can lead to tighter generalization bounds than those of the global Rademacher complexity, this motivates us to consider the following objective function to solve the multi-label learning problem.
\begin{equation}
\min_{W} \sum_{i=1}^{n} \ell(y_{i},f(x_{i},W))+ C\sum_{j>\theta}\lambda_{j}(W),
\end{equation} 
where $\lambda_{j}(W)$ is the $j$-th largest singular value of $W$, and $\theta$ is a parameter to control the tail sum. If we use the squared L2-loss function, we get
\begin{equation}\label{eq:prime}
\min_{W} \sum_{i=1}^{n} \|y_{i}-W^{T}x_{i}\|^{2}+ C\sum_{j>\theta}\lambda_{j}(W).
\end{equation} 

In multi-label learning, the multi-label predictor $W$ usually has a low-rank structure due to the correlations between multiple labels.
The trace norm is regarded as an effective surrogate of rank minimization by simultaneously penalizing all the singular values of $W$. However, it may incorrectly keep the small singular values, which should be zero, or shrink the large singular values to zeros, which should be non-zero. In contrast, our new algorithm can directly minimize over the small singular values, which  encourages the  low-rank structure. 

\begin{figure*}[!htbp]
\begin{center}
   \includegraphics[width=\textwidth]{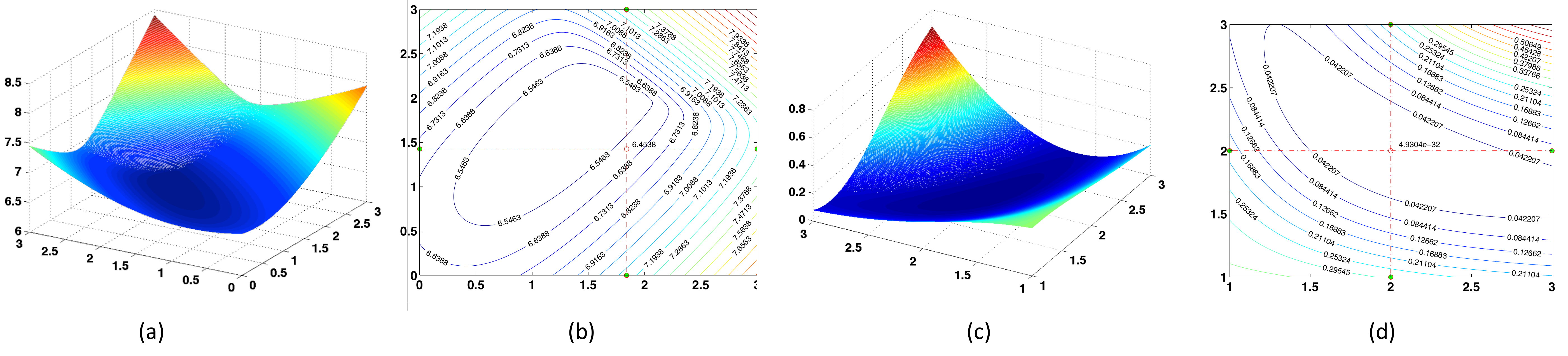}
\end{center}\vskip -0.2in
   \caption{Comparison of the trace norm and the proposed norm as functions of unknown entries $Z_{2,3}$ and $Z_{3,4}$ for the matrix in Eq. (\ref{eq:mat}). (a) and (b) are the 3-D function plot and the contour lines for the trace norm, while (c) and (d) are the same for the proposed norm.}\vspace{-5mm}
\label{fig:mot}
\end{figure*}

To understand why the trace norm may fail in rank minimization, we consider a matrix
\begin{equation}\label{eq:mat}
M = \left[
\begin{array} {lccr}
     2 & 1 & 2 & 1  \\
     1 & 1 & ? & 2  \\
     1 & 1 & 2 & ?
\end{array}
\right],
\end{equation}
where $M_{2,3}$ and $M_{3,4}$ are unknown. The results shown in Figure \ref{fig:mot} plot the trace norm and the proposed new norm of $M$ for all possible completions in a range around the value that minimizes its rank $M_{2,3}=2$ and $M_{3,4}=2$. We find that the trace norm yields the optimal solution for $M$ with singular values $\lambda = [ 5.1235, 1.0338, 0.2965]$ when $M_{2,3} = 1.8377$ and $M_{3,4}=1.4248$. In contrast, we propose to constrain over the tail singular values (setting $\theta = 2$), and derive the optimal solution with singular values $\lambda = [5.3549, 1.1512, 0]$ when $M_{2,3} = 2$ and $M_{3,4}=2$. Hence, the new norm can successfully discover the low-rank structure, while the trace norm fails in this case.

\subsection{Optimization}

Starting with Eq. (\ref{eq:prime}) without the norm regularization, we get the following problem,
\begin{equation}
\min_{W} f(W) = \|Y-W^{T}X\|^{2}_{F},
\end{equation}
where $X$ is the data matrix and $Y$ is the label matrix. The gradient method is a natural approach for solving this problem and generates a sequence of approximate solutions:
\begin{equation}
W_{k} = W_{k-1} - \frac{1}{t_{k}} \nabla f(W_{k-1}),
\end{equation}
which can be further reformulated as a proximal regularization of the linearized function $f(W)$ at $W_{k-1}$ as
\begin{equation}\label{eq:f}
\begin{split}
W_{k} = \min \; &f(W_{k-1}) + \langle W-W_{k-1}, \nabla f(W_{k-1})\rangle \\
    + &\frac{t_{k}}{2}\|W-W_{k-1}\|_{F}^{2}.
\end{split}
\end{equation}
Since Eq. (\ref{eq:f}) can be regarded as a linear approximation of the function $f$ at point $W_{k-1}$ regularized by a quadratic proximal term, Problem (\ref{eq:prime}) can be solved in the following iterative step:
{\scriptsize
\begin{equation} \label{eq:prime_tran}
W_{k} = \min_{W} \frac{t_{k}}{2} \|W-(W_{k-1}-\frac{1}{t_{k}}\nabla f(W_{k-1}))\|_{F}^{2} + C\sum_{j>\theta}\lambda_{j}(W),
\end{equation}
}\noindent
where the terms in Eq. (\ref{eq:f}) that do not depend on $W$ are ignored. 

Recall that if  problem (\ref{eq:prime_tran}) is constrained with the trace norm, \citep{cai2010singular} showed that it can be efficiently solved using the singular value thresholding algorithm. Hence, we propose a  new conditional singular value thresholding algorithm to handle the newly proposed norm regularization.  The solution is summarized in the following theorem.

\begin{theorem}\label{the:opt}
Let $Q\in\mathbb{R}^{m\times n}$ and its SVD decomposition is $Q = U\Sigma V^{T}$, where $U\in \mathbb{R}^{m\times r}$ and $V\in\mathbb{R}^{n\times r}$ have orthonormal columns, $\Sigma\in\mathbb{R}^{r\times r}$ is diagonal. Then,
\begin{equation}
\mathcal{D}^{\theta}(Q) = \arg \min_{W} \{\frac{1}{2} \|W-Q\|_{F}^{2} + C\sum_{j>\theta}\lambda_{j}(W)\}
\end{equation}
is given by $\mathcal{D}^{\theta}(Q) = U\Sigma^{\theta}V^{T}$, where $\Sigma^{\theta}$ is diagonal with $(\Sigma^{\theta})_{ii}= (i\leq \theta \; \& \; \Sigma_{ii}>C) \; ?  \; \Sigma_{ii} : \max(0,\Sigma_{ii}-C)$.
\end{theorem}

\begin{proof}
Assuming that $\widehat{W}$ is the optimal solution, $\mathbf{0}$ should be  a subgradient of the objective function at the point $\widehat{W}$,
\begin{equation}
\mathbf{0} \in \widehat{W} - Q + C \partial(\sum_{j>\theta}\lambda_{j}(\widehat{W})),
\end{equation}
where $\partial(\sum_{j>\theta}\lambda_{j}(\widehat{W}))$ is the set of subgradients of the new norm regularization. Letting $W= U\Sigma V^{T}$, we have
\begin{equation}
\begin{split}
\partial & (\sum_{j>\theta}\lambda_{j}(\widehat{W})) = UI_{\theta}V^{T} + S \\
&s.t. \; U^{T}S = 0, SV = 0, \|S\|\leq 1,
\end{split}
\end{equation}
where $I_{\theta}$ is obtained by setting the  diagonal values  with indices greater than $\theta$ in the identity matrix $I$ as zeros. Set the SVD of $Q$ as
\begin{equation}
Q = U_{0}\Sigma_{0}V_{0}^{T} + U_{1}\Sigma_{1}V_{1}^{T},
\end{equation}
where $U_{0}, V_{0}$ are the singular vectors associated with singular values greater than $C$, while $U_{1}, V_{1}$ correspond to those smaller than or equal to $C$. With these definitions, we have
\begin{equation}
\widehat{W} = U_{0}[\Sigma_{0}-CI_{\theta}]V_{0}^{T},
\end{equation}
and thus,
\begin{equation}
\begin{split}
Q - \widehat{W} = & U_{0}\Sigma_{0}V_{0}^{T} + U_{1}\Sigma_{1}V_{1}^{T} - U_{0}[\Sigma_{0}-CI_{\theta}]V_{0}^{T} \\
& = C (U_{0}I_{\theta}V_{0}^{T} + C^{-1}U_{1}\Sigma_{1}V_{1}^{T}) \\
& = C \partial(\sum_{j>\theta}\lambda_{j}(\widehat{W})),
\end{split} 
\end{equation}
where $S$ is defined as $C^{-1}U_{1}\Sigma_{1}V_{1}^{T}$.
The proof is completed.
\end{proof}
It turns out that the minimization of Problem (\ref{eq:prime_tran}) can be solved by first computing the SVD of $(W_{k-1}-\frac{1}{t_{k}}\nabla f(W_{k-1}))$, and then applying the conditional thresholding on the singular values. By exploiting the structure of the newly proposed norm regularization, the convergence rate of the the resulting algorithm is expected to be the same as that of the gradient method. The whole optimization procedure is shown in Algorithm 1.

\begin{algorithm}[tb]

   \caption{Local Rademacher complexity multi-label learning via  conditional singular value thresholding}
   \label{alg:gbrt}
\begin{algorithmic}
   \STATE {\bfseries Input:} $X, Y, W_{0}, C, \theta, t_{0}, \gamma>1$
   \FOR {$k=1, \cdots, K$}
   
    \STATE $t_{k} = \gamma t_{k-1}$
    \STATE Compute the SVD of $(W_{k-1}-\frac{1}{t_{k}}\nabla f(W_{k-1}))$
    \STATE Update $W_{k}$ by Theorem \ref{the:opt}
    \ENDFOR
   \RETURN $W$
\end{algorithmic}
\end{algorithm}\vskip -0.3in

\section{Experiments}

In this section, we evaluate our proposed algorithm on four datasets from diverse applications: \emph{bibtex}  and \emph{delicious} for tag recommendation,  \emph{yeast} for gene function prediction, and  \emph{corel5k} for image classification. All these datasets were obtained from Mulan's website \footnote{http://mulan.sourceforge.net/datasets.html} and were pre-separated into training and test sets. 
Detailed information about these datasets is shown in Table \ref{tab:dataset}.

\begin{table*}[!htbp]
\caption{Characteristics of the experimental datasets.}
\label{tab:dataset}
\vskip +0.05in
\begin{center}
{\small
\begin{tabular}{lcccccccc}
\hline
  Dataset        & domain &   \# instances & \# attributes & \# labels & cardinality  & density & \# distinct\\
\hline\hline
bibtex            & text       & 7395          & 1836          &  159      & 2.402          & 0.015  & 2856   \\
delicious        & web      & 16105        & 500            &   983     &  19.020       & 0.019  &  15806 \\
yeast              & biology & 2417          &  103           &   14        &  4.237         &  0.303 &  198      \\
corel5k          & images  & 5000          & 499            &   374     &  3.522         & 0.009   & 3175     \\
\hline
\end{tabular}
}
\end{center}
\end{table*}

The details of the competing methods are summarized as follows:
\begin{enumerate}
 \item LRML (\textbf{L}ocal \textbf{R}ademacher complexity \textbf{M}ulti-label \textbf{L}earning): our proposed method. The optimal value of the regularization parameter $C$ was chosen on a validation set.
 \item ML-trace: solves the multi-label learning problem based on Eq. (\ref{eq:erm}) with the squared loss and trace norm.
 \item ML-Fro:  solves the multi-label learning problem based on Eq. (\ref{eq:erm}) with the squared loss and Frobenius  norm.
 \item LEML: the method proposed in \citep{yu2013large}, which decomposes the trace norm into the Frobenius norms of two low-rank matrices.
 \item CPLST: the method proposed in \citep{chen2012feature}, which is equivalent to Problem  \ref{eq:erm} with the low-rank constraint.
\end{enumerate}
Given a test set, we used four criteria  to evaluate the performance of the multi-label predictor:
\begin{itemize}
 \setlength{\itemsep}{1pt}
 \setlength{\parskip}{0pt}
 \setlength{\parsep}{0pt}
  \item Average precision:  evaluates the average fraction of relevant labels ranked higher than a particular label.
 \item Top-K accuracy: for each example,  $K$ labels with the largest decision values for prediction were selected. The average accuracy of all the examples is reported.
 \item Hamming loss: measures the overall classification error through $\frac{1}{nL}\sum_{i}^{n}\sum_{j=1}^{L}\mathcal{I}[round(f^{j}(x_i)) \neq y_{ij}]$.
 \item Average AUC:  the area under the ROC curve for each example was measured, and the average AUC of all the test examples is reported.
\end{itemize}

\begin{table}[!htbp]
\caption{Comparison of LRML with other low-rank aimed algorithms in terms of the different evaluation criteria. $\bullet(\circ)$ indicates that LRML is significantly better (worse) than the corresponding method.}
\label{tab:total}
\vskip -0.2in
\begin{center}
\begin{tabular}{lllllll}
\hline
                                & \multicolumn{6}{c}{Top-1 Accuracy} \\
                                &  \multicolumn{3}{c}{bibtex} & \multicolumn{3}{c}{delicious} \\
  $\theta (k)/L $     & LRML    &   LEML     & CPLST   & LRML    &   LEML     & CPLST  \\
  \hline
$20\%$                   & 0.6002 &   0.5833 $\bullet$ & 0.5555 $\bullet$ & 0.6713 &   0.6716 $\circ$     & 0.6653 $\bullet$ \\
$40\%$                   & 0.6123 &   0.6099 $\bullet$ & 0.5463 $\bullet$ & 0.6708 &   0.6666 $\bullet$ & 0.6625 $\bullet$ \\
$60\%$                   & 0.6330 &   0.6199 $\bullet$ & 0.5753 $\bullet$ & 0.6701 &   0.6628 $\bullet$ & 0.6622 $\bullet$ \\
$80\%$                   & 0.6412 &   0.6394                  & 0.5976 $\bullet$ & 0.6688 &   0.6625 $\bullet$ & 0.6622 $\bullet$ \\
\hline
\hline
                                & \multicolumn{6}{c}{Top-3 Accuracy} \\
                                &  \multicolumn{3}{c}{bibtex} & \multicolumn{3}{c}{delicious} \\
  $\theta (k)/L $     & LRML    &   LEML     & CPLST   & LRML    &   LEML     & CPLST  \\
  \hline
$20\%$                   & 0.3520 &   0.3416 $\bullet$ & 0.3199 $\bullet$ & 0.6297 &   0.6120 $\bullet$ & 0.6113 $\bullet$ \\
$40\%$                   & 0.3713 &   0.3653 $\bullet$ & 0.3453 $\bullet$ & 0.6288 &   0.6123 $\bullet$ & 0.6108 $\bullet$\\
$60\%$                   & 0.3892 &   0.3800                  & 0.3601 $\bullet$ & 0.6300 &   0.6115 $\bullet$ & 0.6109 $\bullet$\\
$80\%$                   & 0.3912 &   0.3858 $\bullet$ & 0.3675 $\bullet$ & 0.6290 &   0.6113 $\bullet$ & 0.6109 $\bullet$\\
\hline
\hline
                                & \multicolumn{6}{c}{Top-5 Accuracy} \\
                                &  \multicolumn{3}{c}{bibtex} & \multicolumn{3}{c}{delicious} \\
  $\theta (k)/L $     & LRML    &   LEML     & CPLST   & LRML    &   LEML     & CPLST  \\
  \hline
$20\%$                   & 0.2511 &   0.2449 $\bullet$& 0.2311 $\bullet$& 0.5690 &   0.5646                 & 0.5630 $\bullet$\\
$40\%$                   & 0.2718 &   0.2684 $\bullet$& 0.2496 $\bullet$& 0.5688 &   0.5639                 & 0.5628 $\bullet$\\
$60\%$                   & 0.2812 &   0.2766 $\bullet$& 0.2607 $\bullet$& 0.5701 &   0.5628 $\bullet$& 0.5623 $\bullet$\\
$80\%$                   & 0.2856 &   0.2820                 & 0.2647 $\bullet$& 0.5703 &   0.5627 $\bullet$& 0.5623 $\bullet$\\
\hline
\hline
                                & \multicolumn{6}{c}{Hamming Loss} \\
                                &  \multicolumn{3}{c}{bibtex} & \multicolumn{3}{c}{delicious} \\
  $\theta (k)/L $     & LRML    &   LEML     & CPLST   & LRML    &   LEML     & CPLST  \\
  \hline
$20\%$                   & 0.0129 &   0.0126 $\circ$   & 0.0127  $\circ$   & 0.0175 &   0.0181 $\bullet$& 0.0182 $\bullet$\\
$40\%$                   & 0.0119 &   0.0124 $\bullet$& 0.0126 $\bullet$& 0.0178 &   0.0181 $\bullet$& 0.0182 $\bullet$\\
$60\%$                   & 0.0118 &   0.0123 $\bullet$& 0.0125 $\bullet$& 0.0178 &   0.0182 $\bullet$& 0.0182 $\bullet$\\
$80\%$                   & 0.0118 &   0.0123 $\bullet$& 0.0125 $\bullet$& 0.0178 &   0.0182 $\bullet$& 0.0182 $\bullet$\\
\hline
\hline
                                & \multicolumn{6}{c}{Average AUC} \\
                                &  \multicolumn{3}{c}{bibtex} & \multicolumn{3}{c}{delicious} \\
  $\theta (k)/L $     & LRML    &   LEML     & CPLST   & LRML    &   LEML     & CPLST  \\
  \hline
$20\%$                   & 0.9023 &   0.8910 $\bullet$& 0.8657 $\bullet$& 0.8812 &   0.8854 $\circ$   & 0.8833  $\circ$\\
$40\%$                   & 0.9056 &   0.9015 $\bullet$& 0.8802 $\bullet$& 0.8897 &   0.8827 $\bullet$& 0.8814   $\bullet$\\
$60\%$                   & 0.9055 &   0.9040                 & 0.8854 $\bullet$& 0.8901 &   0.8814 $\bullet$& 0.8834 $\bullet$\\
$80\%$                   & 0.9049 &   0.9035 $\bullet$& 0.8882 $\bullet$& 0.8895 &   0.8814 $\bullet$& 0.8834 $\bullet$\\
\hline
\end{tabular}
\end{center}\vskip -0.25in
\end{table}

\subsection{Evaluations over Different Norms}

\begin{figure*}[!htbp]
\begin{center}
   \includegraphics[width=\textwidth]{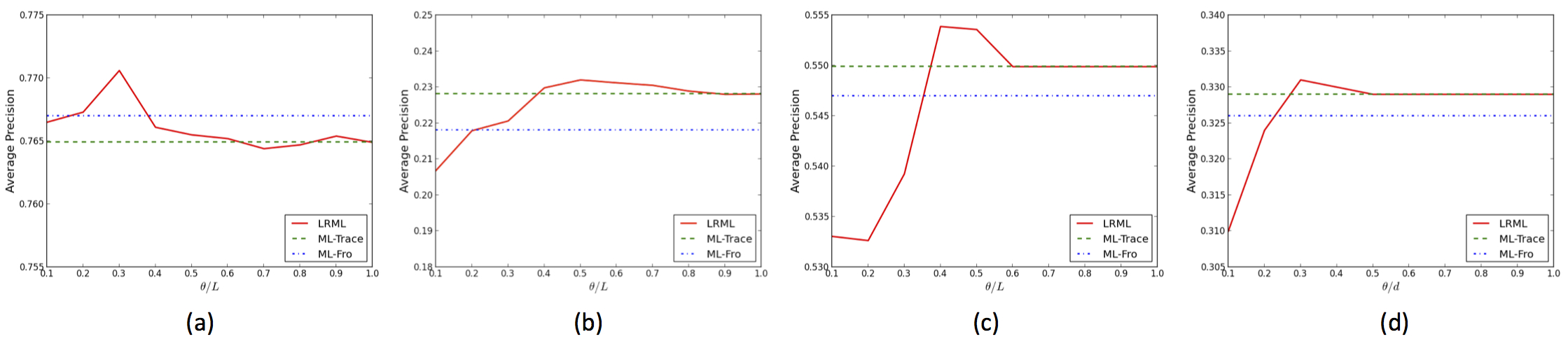}
\end{center}\vskip -0.2in
   \caption{Comparison of  ERM-based multi-learning algorithms with different norms on different datasets: (a) yeast, (b) corel5k, (c) bibtex and (d) delicious.}\vspace{-5mm}
\label{fig:com_norm}
\end{figure*}

We first compared the proposed LRML algorithm with the ML-trace and ML-Fro algorithms on the four datasets; the results are reported in Figure \ref{fig:com_norm}. For LRML, we varied $\theta$ to examine the influence of the number of constrained singular values. As shown in Figure \ref{fig:com_norm},  when $\theta$ is smaller, the performance of LRML is limited, because  the rank of the learned multi-label predictor is too low to cover all the information from different labels.  With  increasing $\theta$, LRML discovers the optimal multi-label predictor with appropriate rank and achieves stable performance. Compared to ML-trace, which constrains over all the singular values, the best LRML performance usually offers further improvements by tightly approximating the rank of the multi-label predictor.  Moreover, it is important to note that ML-Fro is of comparable performance to LRML and ML-trace on the yeast dataset. This is not surprising, since this dataset has a relatively small number of labels, and thus the influence of the low-rank structure is limited.

\subsection{Evaluations over Low-rank Structures}

We next compared LRML algorithm with the current state-of-the-art  LEML and CPLST algorithms. Since all three of these algorithms are implicitly or explicitly designed to exploit the low-rank structure, we accessed their performances using varying dimensionality reduction ratios.
The  results are summarized in Table \ref{tab:total}. 

LRML either improves on, or has comparable performance with, the other methods  for nearly all settings. Although these algorithms  all focus on the low-rank  structure of the predictor in multi-label learning, they study and discover it from different perspectives. LEML factorizes the trace norm by introducing two low-rank matrices, while CPLST learns the multi-label predictor by first learning an encoder to encode the multiple labels. Compared to proposed LRML algorithm, which directly constrains  the tail singular values of the predictor to obtain a low-rank structure, both LEML and CPLST increase the number of free parameters in pursuing the low rank, as a result, the optimization complexity will largely increases with large numbers of labels. In addition, since the objective function of LRML is an explicit minimization of the local Rademacher complexity, it will lead to a tight generalization error bound and guarantee stable performances for unseen examples.

In order to investigate the convergence behaviors of LRML, we plot the objective values of LRML on the yeast dataset in Figure \ref{fig:con}. We can observe that LRML converges fast in different cases. This confirms that the proposed conditional singular value thresholding algorithm can efficiently solve LRML.

\begin{figure}[!htbp]
\begin{center}
   \includegraphics[width=0.8\columnwidth]{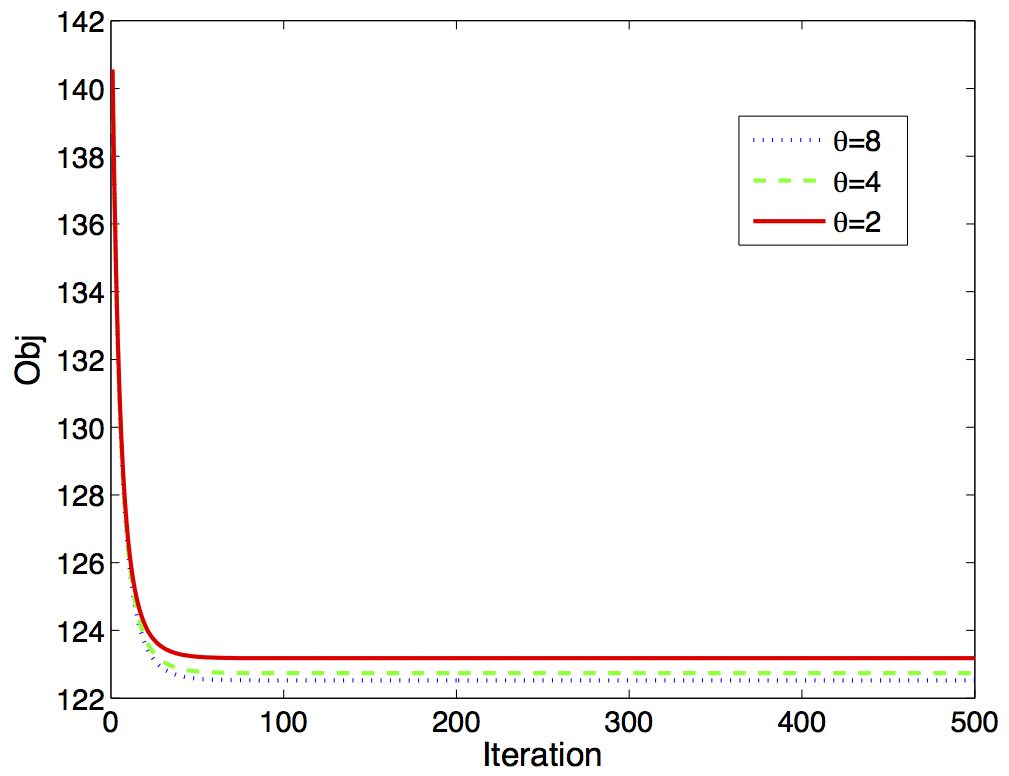}
\end{center}\vskip -0.2in
   \caption{The convergence curve of LRML on the yeast dataset.}\vspace{-5mm}
\label{fig:con}
\end{figure}

\section{Conclusion}

In this paper, we use the principle of local Rademacher complexity to guide the design of a new multi-label learning algorithm. We analyze the local Rademacher complexity of ERM-based multi-label learning algorithms, and discover that it is upper-bounded by the tail sum of the singular values of the multi-label predictor. Inspired by this local Radermacher complexity bound, a new multi-label learning algorithm is therefore proposed that  concentrates solely on the tail singular values of the predictor, rather than on all the singular values as with the trace norm.  
This use of the local Rademacher complexity results in a sharper generalization error bound and moreover, the new constraints over tail singular values provides a  tighter approximation of the low-rank structure than the trace norm. The experimental results  demonstrate the effectiveness of the proposed algorithm in discovering the low-rank structure and its  generalizability for multi-label prediction.

\vskip 0.2in
\bibliography{LRML_arxiv.bbl}

\begin{thebibliography}{25}
\providecommand{\natexlab}[1]{#1}
\providecommand{\url}[1]{\texttt{#1}}
\expandafter\ifx\csname urlstyle\endcsname\relax
  \providecommand{\doi}[1]{doi: #1}\else
  \providecommand{\doi}{doi: \begingroup \urlstyle{rm}\Url}\fi

\bibitem[Amit et~al.(2007)Amit, Fink, Srebro, and Ullman]{amit2007uncovering}
Yonatan Amit, Michael Fink, Nathan Srebro, and Shimon Ullman.
\newblock Uncovering shared structures in multiclass classification.
\newblock In \emph{Proceedings of the 24th international conference on Machine
  learning}, pages 17--24. ACM, 2007.

\bibitem[Bartlett and Mendelson(2003)]{bartlett2003rademacher}
Peter~L Bartlett and Shahar Mendelson.
\newblock Rademacher and gaussian complexities: Risk bounds and structural
  results.
\newblock \emph{The Journal of Machine Learning Research}, 3:\penalty0
  463--482, 2003.

\bibitem[Bartlett et~al.(2005)Bartlett, Bousquet, and
  Mendelson]{bartlett2005local}
Peter~L Bartlett, Olivier Bousquet, and Shahar Mendelson.
\newblock Local rademacher complexities.
\newblock \emph{Annals of Statistics}, pages 1497--1537, 2005.

\bibitem[Barutcuoglu et~al.(2006)Barutcuoglu, Schapire, and
  Troyanskaya]{barutcuoglu2006hierarchical}
Zafer Barutcuoglu, Robert~E Schapire, and Olga~G Troyanskaya.
\newblock Hierarchical multi-label prediction of gene function.
\newblock \emph{Bioinformatics}, 22\penalty0 (7):\penalty0 830--836, 2006.

\bibitem[Bi and Kwok(2014)]{bi2014multilabel}
Wei Bi and James~T Kwok.
\newblock Multilabel classification with label correlations and missing labels.
\newblock In \emph{Proceedings of AAAI Conference on Artificial Intelligence
  (AAAI)}, 2014.

\bibitem[Cabral et~al.(2011)Cabral, Torre, Costeira, and
  Bernardino]{cabral2011matrix}
Ricardo~S Cabral, Fernando Torre, Jo{\~a}o~P Costeira, and Alexandre
  Bernardino.
\newblock Matrix completion for multi-label image classification.
\newblock In \emph{Advances in Neural Information Processing Systems}, pages
  190--198, 2011.

\bibitem[Cai et~al.(2010)Cai, Cand{\`e}s, and Shen]{cai2010singular}
Jian-Feng Cai, Emmanuel~J Cand{\`e}s, and Zuowei Shen.
\newblock A singular value thresholding algorithm for matrix completion.
\newblock \emph{SIAM Journal on Optimization}, 20\penalty0 (4):\penalty0
  1956--1982, 2010.

\bibitem[Chen and Lin(2012)]{chen2012feature}
Yao-Nan Chen and Hsuan-Tien Lin.
\newblock Feature-aware label space dimension reduction for multi-label
  classification.
\newblock In \emph{Advances in Neural Information Processing Systems}, pages
  1529--1537, 2012.

\bibitem[Doppa et~al.(2014)Doppa, Yu, Ma, Fern, and Tadepalli]{doppa2014hc}
Janardhan~Rao Doppa, Jun Yu, Chao Ma, Alan Fern, and Prasad Tadepalli.
\newblock Hc-search for multi-label prediction: An empirical study.
\newblock In \emph{Proceedings of AAAI Conference on Artificial Intelligence
  (AAAI)}, 2014.

\bibitem[Gao et~al.(2004)Gao, Wu, Lee, and Chua]{gao2004mfom}
Sheng Gao, Wen Wu, Chin-Hui Lee, and Tat-Seng Chua.
\newblock A mfom learning approach to robust multiclass multi-label text
  categorization.
\newblock In \emph{Proceedings of the twenty-first international conference on
  Machine learning}, page~42. ACM, 2004.

\bibitem[Guo and Xue(2013)]{guo2013probabilistic}
Yuhong Guo and Wei Xue.
\newblock Probabilistic multi-label classification with sparse feature
  learning.
\newblock In \emph{Proceedings of the Twenty-Third international joint
  conference on Artificial Intelligence}, pages 1373--1379. AAAI Press, 2013.

\bibitem[Hariharan et~al.(2010)Hariharan, Zelnik-Manor, Varma, and
  Vishwanathan]{hariharan2010large}
Bharath Hariharan, Lihi Zelnik-Manor, Manik Varma, and Svn Vishwanathan.
\newblock Large scale max-margin multi-label classification with priors.
\newblock In \emph{Proceedings of the 27th International Conference on Machine
  Learning (ICML-10)}, pages 423--430, 2010.

\bibitem[Huang et~al.(2012)Huang, Zhou, and Zhou]{huang2012multi}
Sheng-Jun Huang, Zhi-Hua Zhou, and ZH~Zhou.
\newblock Multi-label learning by exploiting label correlations locally.
\newblock In \emph{AAAI}, 2012.

\bibitem[Li and Guo(2013)]{li2013active}
Xin Li and Yuhong Guo.
\newblock Active learning with multi-label svm classification.
\newblock In \emph{Proceedings of the Twenty-Third international joint
  conference on Artificial Intelligence}, pages 1479--1485. AAAI Press, 2013.

\bibitem[Loeff and Farhadi(2008)]{loeff2008scene}
Nicolas Loeff and Ali Farhadi.
\newblock Scene discovery by matrix factorization.
\newblock In \emph{Computer Vision--ECCV 2008}, pages 451--464. Springer, 2008.

\bibitem[Read et~al.(2011)Read, Pfahringer, Holmes, and
  Frank]{read2011classifier}
Jesse Read, Bernhard Pfahringer, Geoff Holmes, and Eibe Frank.
\newblock Classifier chains for multi-label classification.
\newblock \emph{Machine learning}, 85\penalty0 (3):\penalty0 333--359, 2011.

\bibitem[Tai and Lin(2012)]{tai2012multilabel}
Farbound Tai and Hsuan-Tien Lin.
\newblock Multilabel classification with principal label space transformation.
\newblock \emph{Neural Computation}, 24\penalty0 (9):\penalty0 2508--2542,
  2012.

\bibitem[Tsoumakas et~al.(2010)Tsoumakas, Katakis, and
  Vlahavas]{tsoumakas2010mining}
Grigorios Tsoumakas, Ioannis Katakis, and Ioannis Vlahavas.
\newblock Mining multi-label data.
\newblock In \emph{Data mining and knowledge discovery handbook}, pages
  667--685. Springer, 2010.

\bibitem[Vapnik(1998)]{vapnik1998statistical}
Vapnik.
\newblock \emph{Statistical learning theory}, volume~2.
\newblock Wiley New York, 1998.

\bibitem[Wang et~al.(2009)Wang, Yan, Zhang, and Zhang]{wang2009multi}
Changhu Wang, Shuicheng Yan, Lei Zhang, and Hong-Jiang Zhang.
\newblock Multi-label sparse coding for automatic image annotation.
\newblock In \emph{Computer Vision and Pattern Recognition, 2009. CVPR 2009.
  IEEE Conference on}, pages 1643--1650. IEEE, 2009.

\bibitem[Xu et~al.(2013{\natexlab{a}})Xu, Jin, and Zhou]{xu2013speedup}
Miao Xu, Rong Jin, and Zhi-Hua Zhou.
\newblock Speedup matrix completion with side information: Application to
  multi-label learning.
\newblock In \emph{Advances in Neural Information Processing Systems}, pages
  2301--2309, 2013{\natexlab{a}}.

\bibitem[Xu et~al.(2013{\natexlab{b}})Xu, Yu-Feng, and Zhi-Hua]{zhi2013multi}
Miao Xu, Li~Yu-Feng, and Zhou Zhi-Hua.
\newblock Multi-label learning with pro loss.
\newblock In \emph{Proceedings of AAAI Conference on Artificial Intelligence
  (AAAI)}, 2013{\natexlab{b}}.

\bibitem[Yu et~al.(2014)Yu, Jain, and Dhillon]{yu2013large}
Hsiang-Fu Yu, Prateek Jain, and Inderjit~S Dhillon.
\newblock Large-scale multi-label learning with missing labels.
\newblock In \emph{Proceedings of the twenty-first international conference on
  Machine learning}, 2014.

\bibitem[Zhang and Zhou(2013)]{zhang2013review}
M~Zhang and Z~Zhou.
\newblock A review on multi-label learning algorithms.
\newblock 2013.

\bibitem[Zhang and Zhang(2010)]{zhang2010multi}
Min-Ling Zhang and Kun Zhang.
\newblock Multi-label learning by exploiting label dependency.
\newblock In \emph{Proceedings of the 16th ACM SIGKDD international conference
  on Knowledge discovery and data mining}, pages 999--1008. ACM, 2010.

\end{thebibliography}

\end{document}